\documentclass[letterpaper, 10 pt, conference]{ieeeconf}  

\usepackage{algorithm}
\usepackage{algorithmic}
\usepackage{custom}
\usepackage{amsmath}
\usepackage{amssymb}
\usepackage{exsheets}
\usepackage{fancyhdr}
\usepackage{comment}
\newcommand{\norm}[1]{\left\lVert#1\right\rVert}
\newtheorem{theorem}{Theorem}
\newtheorem{remark}{Remark}
\newtheorem{lemma}{Lemma}
\newtheorem{corollary}{Corollary}
\newtheorem{assumption}{Assumption}

\IEEEoverridecommandlockouts                              

\title{
Reinforcement Learning with Unbiased Policy Evaluation and Linear Function Approximation
}

\author{Anna Winnicki$^{1}$ and R. Srikant$^{1}$
\thanks{$^{1}$Anna Winnicki and R. Srikant are with Coordinated Science Laboratory and the Department of Electrical and Computer Engineering at
        University of Illinois at Urbana-Champaign. Srikant is also with the C3.ai Digital Transformation Institute.
        {\tt\small annaw5@illinois.edu,rsrikant@illinois.edu}.}%
}
\begin{document}

\maketitle
\thispagestyle{empty}
\pagestyle{empty}

\begin{abstract}
We provide performance guarantees for a variant of simulation-based policy iteration for controlling Markov decision processes that involves the use of stochastic approximation algorithms along with state-of-the-art techniques that are useful for very large MDPs, including lookahead, function approximation, and gradient descent. Specifically, we analyze two algorithms; the first algorithm involves a least squares approach where a new set of weights associated with feature vectors is obtained via least squares minimization at each iteration and the second algorithm involves a two-time-scale stochastic approximation algorithm taking several steps of gradient descent towards the least squares solution before obtaining the next iterate using a stochastic approximation algorithm. 
\end{abstract}

\section{Introduction}
We study the problem of controlling stochastic systems with simulation-based methods. In particular, we want to find an optimal control policy to minimize the expected cost in a discrete-time Markov decision process (MDP). By simulation-based methods, we mean methods which evaluate the performance of a policy by estimating the value function of a few states under that policy by observing trajectories of the underlying Markov chain starting from each of those states.

When the size of the state space of the Markov decision process is very large, for example in chess and Go, commonly used algorithms such as policy iteration are infeasible as they involve expensive computations at each iteration. So, modern state-of-the-art algorithms \cite{silver2016mastering, silver2017mastering,silver2017shoji} utilize simulation-based variants of policy iteration on top of techniques such as function approximation, lookahead, and gradient descent to yield fast convergence in  problems with large state spaces. See Section \ref{sectionprelimunbiased} or the work in \cite{bertsekas2019reinforcement} for more on lookahead.

The use of simulation-based variants of policy iteration has been studied in the work in \cite{tsitsiklis2002convergence} using results from stochastic approximation to guarantee convergence of the algorithms. However, the work in \cite{tsitsiklis2002convergence} assumes that it is possible to obtain unbiased estimates of value functions for each state of the state space. In this paper, we analyze variants of policy iteration where simulations of value functions corresponding to lookahead policies are only available for some of the states at each iteration.  
We consider several techniques employed in state-of-the-art techniques to improve the computational efficiency of reinforcement learning algorithms for MDPs with very large state spaces including lookahead policies, linear function approximation, and gradient descent. 

Our contributions are as follows:
\begin{itemize}
\item We study two simulation-based approximate policy iteration algorithms with lookahead that employ stochastic approximation techniques and linear function approximation. In both algorithms, at each iteration, we obtain unbiased estimates of the value function corresponding to a lookahead policy for several states. Further, it is assumed that the value function associated with each state is a linear function of a feature vector associated with the state. In the first algorithm, we use a least squares approach based on the unbiased estimates to obtain weights associated with the feature vectors to estimate the value function corresponding to the lookahead policy for all states. Then, we update the estimate of the value function in a manner similar to the work of \cite{tsitsiklis2002convergence} via the Robbins-Monro algorithm \cite{bertsekastsitsiklis, robbins1951stochastic}. In the second algorithm, instead of computing the least squares minimizer, we instead employ a two-time-scale stochastic approximation algorithm where we take several steps of gradient descent towards the solution of the least squares problem and use the result as the estimate for the value function in our stochastic approximation algorithm, the same way the minimizer is used in the least squares algorithm. For both algorithms, we derive upper bounds on the approximation error. Our performance guarantees and proofs for the least squares and gradient descent algorithms can be found in Sections \ref{sectionlsunbiased} and \ref{sectionGDunbiased}, respectively.

\item  We obtain performance guarantees for the algorithms that depend only on the amount of lookahead and the choice of feature vectors. We show that the upper bounds on the errors decrease exponentially with the amount of lookahead up to a constant factor, analogously to work of \cite{winnicki}. However, the results in this paper give tighter bounds by exploiting the unbiasedness of policy evaluation. Thus, when the feature vectors do not accurately estimate the value functions, a possible remedy involves increasing the amount of lookahead. It is important to note that the error in the algorithms depends not on the size of the state space but rather on the choice of feature vectors. 

\item As mentioned earlier, our results generalize the results in \cite{tsitsiklis2002convergence}. As we will remark later, in the special case where the standard greedy policy is used in the policy improvement step of the algorithm, instead of a multi-step lookahead, and if we do not use function approximation (i.e., what is known as the tabular case in the reinforcement learning literature), then we can recover the results in \cite{tsitsiklis2002convergence}. In the more general case considered in this paper, our results show that the asymptotic error depends on the quality of the function approximation, which can be further mitigated by using lookahead policies.
\end{itemize}

We now discuss the relationship between the results in the paper and prior work. Reinforcement learning algorithms can be broadly categorized into two classes:
\begin{itemize}
    \item Algorithms such as Q-learning and policy gradient algorithms which are designed for systems where the underlying model is unknown \cite{abadcontrol,controltsitsiklis,mehta2009q,mathkar2016distributed}.
    \item Algorithms for problems in which the model is known but the state-space is discrete and very large, which is the focus of this paper.
\end{itemize}
As mentioned earlier, the work in \cite{tsitsiklis2002convergence} considers the special case when value function estimates are obtained for each state at each iteration. The convergence analysis in \cite{tsitsiklis2002convergence} is further motivated by similar analysis for modified policy iteration \cite{Puterman1978ModifiedPI}. The role of lookahead and its relationship to Model Predictive Control (MPC) has been discussed in the recent book \cite{bertsekas2021lessons} but to the best of our knowledge, no convergence analysis of the type considered in this paper is provided there. 

The class of algorithms considered in this paper fall within the class of approximate policy iteration algorithms which have been extensively studied; see \cite{bertsekastsitsiklis, bertsekas2019reinforcement, Puterman1978ModifiedPI, scherrer,efroni2020online, tomar2020multistep, efroni2018multiplestep, 9407870}, for example. In particular, the works of \cite{Bertsekas2011ApproximatePI} and \cite{bertsekas2019reinforcement} consider the use of feature vectors in a variant of policy iteration. However, the use of gradient descent, lookahead and unbiased policy evaluation are not considered there.

Closely related to this paper is our earlier work \cite{winnicki}, which analyzes the role of linear function approximation, lookahead, and gradient descent in the convergence of modified policy iteration (also known as optimistic policy iteration) with noise, where at each iteration partial simulations of the value function corresponding to the lookahead policy are obtained for several states, i.e., the Bellman operator associated with a policy is applied several times starting at some of the states. When the Bellman operator is applied $m$ times to a vector, we have an $m$-return corresponding to the vector. Then, either the least squares problem or a step in the gradient descent towards the solution of the least squares problem is obtained to get weights associated with the feature vectors in the function approximation of the value function. The asymptotic error bounds there are weaker than those provided in this paper due to the fact that the unbiasedness of the value function estimates are not taken into account in that paper. However, the results there do not require a full simulation rollout. Thus, the results in the two papers complement each other. We also considered unbiased policy evaluation in \cite{winnicki2} but the results there are restricted to the case of MDPs with certain specific graph structures. The results here do not require such an assumption and further, consider much more general function approximation schemes and gradient descent which are not considered there.

We also note that if one performs policy evaluation with a partial rollout (also called $m$-step return), then the resulting algorithm may not have bounded errors \cite{tsitsiklisvanroy,winnicki}. In this paper, we assume a full rollout, so this issue does not arise.

\section{Preliminaries} \label{sectionprelimunbiased}
We consider a finite-state finite-action Markov decision process (MDP). We denote by $\scriptS$ our state space where $|\scriptS|$ is the size of the state space and $\scriptA$ our action space where at every state $i \in \scriptS,$ an action $u \in \scriptA$ may be taken. The probability of transitioning to state $j \in \scriptS$ from state $i \in \scriptS$ when action $u \in \scriptA$ is taken is denoted by $P_{ij}(u)$. Each time we take an action $u \in \scriptA$ at state $i \in \scriptS$, we incur a non-deterministic cost $g(i, u).$ We assume the following about our costs:
\begin{assumption}\label{assume 1} 
$g(i,u) \in[0,1]$ $\forall i,u,$ with probability 1.

\hfill $\diamond$
\end{assumption} 

A policy $\mu: \scriptS \to \scriptA$ is defined to be a mapping from the state space to the action space which prescribes an action to take when the Markov decision process reaches a particular state $i \in \scriptS.$ When a policy $\mu$ is fixed, we denote by $g_{\mu} \in \mathbb{R}^{|\scriptS|}$ the vector of expected costs associated with policy $\mu,$ i.e., $g_{\mu}(i) := E[g(i, \mu(i))].$ 
We denote by $P_\mu$ the probability transition matrix for the associated Markov chain. In other words, $P_{\mu}(i, j) := P_{ij}(\mu(i)) \forall i, j \in \scriptS$. At time $k$, we denote the state of the Markov decision process $x_k.$ For a given policy $\mu,$ we define the value function, $J^{\mu},$ with discount factor $\alpha \in (0, 1)$, component-wise as follows:
\begin{align*}
    J^{\mu}(i) := E[\sum_{k=0}^\infty \alpha^k g(x_k, \mu(x_k))|x_0 = i] \quad \forall i \in \scriptS.
\end{align*}
It is well known that $J^\mu$ can be obtained by solving the associated Bellman equation:
\begin{align*}
        J^\mu = g_\mu + \alpha P_\mu J^\mu.
\end{align*} 

For any vector $J$, we define the operator $T_\mu: \scriptS \to \scriptS$ as follows:
\begin{align*}
    T_\mu J = g_\mu + \alpha P_\mu J. 
\end{align*} Thus, the solution of the Bellman equation corresponding to policy $\mu$ is the fixed point $J^\mu$ of $J^\mu = T_\mu J^\mu.$

Our objective is to find the policy which minimizes the expected discounted cost with discount factor $\alpha \in (0, 1)$. In other words, we seek a policy $\mu$ which minimizes the following:
\begin{align*}
    E[\sum_{k=0}^\infty \alpha^k g(x_k, \mu(x_k))|x_0 = i] \quad \forall i \in \scriptS. \label{eq:objective unbiased}
\end{align*}

We call the value function associated with this policy $J^*$ and $J^*$ will be referred to as the optimal value function. That is,
\begin{align*}
    J^* := \min_\mu J^\mu.
\end{align*} 
Written differently, we have that:
\begin{align*}
    J^*(i) =  \min_\mu E[\sum_{k=0}^\infty \alpha^k g(x_k, \mu(x_k))|x_0 = i].
\end{align*}

In order to find $J^*$ and a corresponding optimal policy, we define the Bellman optimality operator $T$. When the context is understood, we use the term Bellman operator to denote $T$. Consider any vector $J$. We define the Bellman operator $T: \mathbb{R}^{|\scriptS|} \to \mathbb{R}^{|\scriptS|}$ as follows:  
\begin{align}
TJ &= \min_\mu E[ g_\mu + \alpha P_\mu J].
\end{align}
Component-wise, we have the following:
\begin{align*}
TJ(i) &= \underset{u}\min \Big [ g(i, u) + \alpha \sum_{j=1}^{|\scriptS|} P_{ij}(u)J(j) \Big ]. \label{T}
\end{align*}
The policy corresponding to the $T$ operator is defined as the \textit{greedy} policy. When applied $H$ times to vector $J \in \mathbb{R}^{|\scriptS|}$, we call the resulting operator, $T^H,$ the $H$-step ``lookahead'' corresponding to $J$. We call the greedy policy corresponding to $T^H J$ is called the $H$-step lookahead policy, or the lookahead policy, when $H$ is understood. For a formal definition of the lookahead policy, see Section \ref{sectionlsunbiased} or the work of \cite{winnicki}.

It is well known that each time the Bellman operator is applied to a vector $J$ to obtain $TJ,$ the following holds:
\begin{align*}
    \norm{TJ-J^*}_\infty\leq \alpha\norm{J-J^*}_\infty.
\end{align*} 

The Bellman equations state that $J^*$ is a solution to
\begin{align*}
J^* = TJ^*.
\end{align*}
Note that every greedy policy with respect to the optimal value function $J^*$ is optimal and vice versa \cite{bertsekastsitsiklis}. 

We now state several well known useful properties of the operators $T$ and $T_\mu$. We consider the vector $e \in \mathbb{R}^{|\scriptS|}$ where $e(i) = 1 \forall i \in 1, 2, \ldots, |\scriptS|.$ The following holds:
\begin{equation}
    T(J + ce) = TJ + \alpha ce, \quad T_\mu(J + ce) = T_\mu J + \alpha ce. \label{eq:usefulproperties}
\end{equation}
Operators $T$ and $T_\mu$ are monotone operators:
\begin{align}
    J \leq J' \implies TJ \leq TJ', \quad T_\mu J \leq T_\mu J'. \label{monotonicityproperty}
\end{align}

\section{Least Squares Algorithm}

The algorithm that we study in this section is presented in Algorithm~\ref{alg:LSUnbiased} which we discuss next.
\begin{algorithm}
\caption{Least Squares Function Approximation Algorithm With Unbiased Noise and Lookahead}
\label{alg:LSUnbiased}
\textbf{Input}: $V_0,$ feature vectors $\{ \phi(i) \}_{i \in \scriptS}, \phi(i) \in \mathbb{R}^d$  and subsets $\scriptD  \subseteq \scriptS, k = 0, 1, \ldots.$ Here $\scriptD $ is the set of states at which we evaluate the current policy at iteration $k.$\\
\begin{algorithmic}[1] 
\STATE Let $k=0$.
\STATE Let $\mu_{k+1}$ be such that $T_{\mu_{k+1}}T^{H-1}V_k = T^H V_k$.\\
\STATE Compute $\hat{J}^{\mu_{k+1}}(i) =  J^{\mu_{k+1}}(i)+w_{k}(i)$ for $i \in \scriptD .$ \\ \label{step 3 alg}
\STATE Choose $\theta_{k+1}$ to solve 
\begin{align}
    \min_\theta \sum_{i \in \scriptD} \Big( (\Phi \theta)(i) - \hat{J}^{\mu_{k+1}}(i) \Big)^2, \label{eq:step 4 alg}
\end{align} \\
where $\Phi$ is a matrix whose rows are the feature vectors.
\STATE \begin{align}
    V_{k+1} = (1-\gamma_k)V_k + \gamma_k (\Phi \theta_{k+1}). \label{eq:iteratesLSUnbiased} 
\end{align}
\STATE Set $k \leftarrow k+1.$ Go to 2.
\end{algorithmic}
\end{algorithm}

\label{sectionlsunbiased}
At every iteration, we have an iterate $V_k.$ We determine an \color{black} \textit{$H$-step lookahead policy}, \color{black}$\mu_{k+1},$ that is obtained as follows:
\begin{align*}
    T^H J_k = T_{\mu_{k+1}} T^{H-1} J_k.
\end{align*}

An $H$-step lookahead policy can be computationally expensive to implement, so an approximation called Monte Carlo Tree Search is used in practice and is known to approximate the lookahead policy well \cite{chang2005adaptive,kocisszepesvari}. 

We wish to estimate  $J^{\mu_{k+1}}$. We will obtain unbiased estimates of $J^{\mu_{k+1}}(i)$ for only a select few states $i \in \scriptD$. Our assumption on the set $\scriptD$ where $|\scriptD| <<|\scriptS|$ is discussed in more detail later in this section.  For states $i \in \scriptD$, we simulate a trajectory of $\sum_{k=0}^\infty \alpha^k g(x_k, \mu_{k+1}(x_k)),$ where $x_0 = i$ which we call $\hat{J}^{\mu_{k+1}}(i).$
Since $\hat{J}^{\mu_{k+1}}(i)$ is an unbiased estimate of $J^{\mu_{k+1}}(i),$ we write $\hat{J}^{\mu_{k+1}}(i) = J^{\mu_{k+1}}(i)+w_k(i)$ for states $i \in \scriptD.$ Note that $w_k(i)$ for $i\notin\scriptD $ does not affect the algorithm, so for convenience we define $w_k(i)=0$ for $i\notin\scriptD .$ 

We let $\scriptF_k$ be the filtration that denotes the history of the noise before $w_k$ has been determined. In other words, $\scriptF_k$ is defined as follows:
\begin{align}
    \scriptF_k := \{
(w_{k'})_{k'\leq k-1}\}. \label{eq:filtration}
\end{align}

Observe that $E[w_k|\scriptF_k]=0.$ Furthermore, from Assumption \ref{assume 1} and since $\alpha \in (0, 1),$ we have that $E[\norm{w_k}_\infty|\scriptF_k] \leq \frac{1}{1-\alpha}.$ 

We summarize the above with the following assumption:

\begin{assumption}
    $E[w_k|\scriptF_k]= 0$ and $E[\norm{w_k}_\infty|\scriptF_k] \leq \frac{1}{1-\alpha}.$ \label{assume 2}
    
\hfill $\diamond$
\end{assumption}

We associate with each state $i \in \scriptS$ a feature vector $\phi(i) \in \mathbb{R}^d$ in order to obtain an estimate of $J^{\mu_{k+1}}(i)$ for all $i \in \scriptS$. The matrix of feature vectors is $\Phi \in \mathbb{R}^{|\scriptS|\times d}$, where each row of $\Phi$ is feature vector corresponding to a state. The estimates of $J^{\mu_{k+1}}(i)$ for $i \in \scriptD$ are used to obtain estimates of $J^{\mu_{k+1}}(i)$ for all $i \in \scriptS.$ In order to obtain the estimates for states $i \notin \scriptD,$ we associate with each state $i \in \scriptS$ a feature vector $\phi(i) \in \mathbb{R}^d$. We then obtain a $\theta_{k+1}$ as follows: 
\begin{align}
    \theta_{k+1} = \min_\theta \sum_{i \in \scriptD} \Big( (\Phi \theta)(i) - \hat{J}^{\mu_{k+1}}(i) \Big)^2.
\end{align} 
Note that $\theta_{k+1}$ may be ill-defined. Thus, we make the following assumption which states that we explore a sufficient number of states during the policy evaluation phase at each iteration.

\begin{assumption}\label{assume 3} 
$ \text{ rank }\{ \phi(i)\}_{i \in \scriptD} = d$.

\hfill $\diamond$
\end{assumption}
Our estimate for $J^{\mu_{k+1}}$ is given by $\Phi \theta_{k+1}$. 

We have that $\theta_{k+1}$ can be written as:
\begin{align}
\theta_{k+1} &= (\Phi_{\scriptD }^\top \Phi_{\scriptD } )^{-1} \Phi_{\scriptD }^\top \scriptP  \hat{J}^{\mu_{k+1}},\label{eq:rewritetheta}\end{align} 
where $\Phi_{\scriptD }$ is a matrix whose rows are the feature vectors of the states in $\scriptD $ and $\scriptP $ is a matrix of zeros and ones such that $\scriptP \hat{J}^{\mu_{k+1}}$ is a vector whose elements are a subset of the elements of $\hat{J}^{\mu_{k+1}}$ corresponding to $\scriptD $.

We obtain our next iterate as follows:
\begin{align*}
    V_{k+1} &= (1-\gamma_k)V_k + \gamma_k (\Phi \theta_{k+1}).
\end{align*}
Thus, $\scriptF_k$ is the filtration that denotes the history of the algorithm up to and including the point where $V_k$ can be computed, but before $w_k$ has been determined. 

We will use standard stochastic approximation results such as those employed in Proposition 4.4 of \cite{bertsekastsitsiklis} to establish convergence of our iterates $V_k$ in Theorem \ref{UnbiasedLSThm}. For completeness, we summarize the result we use below: 

\begin{lemma}\label{Lemma 1 Unbiased} 
Let $H: \mathbb{R}^n \to \mathbb{R}^n$ be a sup norm contraction, i.e., an operator such that 
\begin{align*}
    \norm{H(x)-H(y)}_\infty \leq C\norm{x-y}_\infty,
\end{align*} where $C \in (0, 1).$ Given iterates $\{V_k\}_{k=0}^\infty,$ where $V_k \in \mathbb{R}^n,$  where each component $V_k(i)$ of $V_k$ is generated by the following:
$$V_{k+1}(i) = (1-\gamma_k(i))V_k(i) + \gamma_k(i)(H V_k(i) + w_{k}(i)), 
$$ for all $k=0, 1, \ldots,$ and $w_k(i)$ is a random noise term. The filtration for this algorithm is given by $\scriptF_k =\{ V_0(i), \ldots, V_k(i), w_0(i), \ldots, w_{k-1}(i), i = 1, \ldots, n\}$. Assume the following conditions:
	\begin{enumerate}
		\item The stepsizes $\gamma_k(i)$ are non-negative and satisfy 
		\begin{align*}
		    \sum_{k=0}^\infty \gamma_k(i) = \infty, \quad \sum_{k=0}^\infty \gamma_k^2(i) < \infty;
		\end{align*}
		\item The noise terms are unbiased conditioned on the past and their conditional variance is bounded, i.e.,
		\begin{itemize}
		    \item For every $i$ and $k$, we have $E[w_k(i)|\scriptF_k]=0.$
		    \item Given any norm $\norm{\cdot}$ on $\mathbb{R}^n,$ there exist constants $A$ and $B$ such that 
		        \begin{align*}
		            E[w_k^2(i)|\scriptF_k]\leq A+B\norm{V_k}^2, \forall i,k;
		        \end{align*}
		\end{itemize}
		\item $H: \mathbb{R}^n \to \mathbb{R}^n$ is a weighted maximum norm pseudo-contraction with fixed point $V^*,$ i.e.,
		    \begin{align*}
		        \norm{HV - V^*}_\xi \leq \alpha \norm{V-V^*}_\xi, \quad \forall V,
		    \end{align*} where $\norm{V}_\xi := \max_i \frac{|V(i)|}{\xi(i)}$ for some $\alpha \in [0, 1)$ and vector $\xi.$
	\end{enumerate}
Then, $V_k$ converges to $V^*,$ with probability 1.

\end{lemma}

We will use a special case of Lemma \ref{Lemma 1 Unbiased} to obtain convergence bounds on $V_k$ where $\gamma_k$ does not depend on states $i \in \scriptS$, i.e., $\gamma_k(i)=\gamma_k(j) \forall i, j \in \scriptS$. To do so, we make the following assumption concerning $\gamma_k$ in equation \eqref{eq:iteratesLSUnbiased}:
\begin{assumption} \label{assume 4}
\begin{align*}\sum_{i=1}^\infty \gamma_k = \infty, \quad \sum_{i=1}^\infty \gamma_k^2 < \infty.\end{align*}

\hfill $\diamond$
\end{assumption}

Before we state the main result of this section, we present a lemma and a corollary which will be used in the proof of the theorem.
\begin{lemma}
For any policies $\tilde{\mu}$ and $\mu,$ we have the following:
\begin{align*}
    & \norm{ T_{\tilde{\mu}}\scriptM J^{\mu} - \scriptM J^{\mu}}_\infty  
    \leq  (\frac{1}{1-\alpha}+\delta_2)(1+\alpha).
\end{align*} \label{lemma 2 unbiased}
\end{lemma}
\begin{proof}
\begin{align*}
    & \norm{ T_{\tilde{\mu}}\scriptM J^{\mu} - \scriptM J^{\mu}}_\infty  \allowdisplaybreaks\\
    \allowdisplaybreaks\\
    &=  \norm{ T_{\tilde{\mu}}\scriptM J^{\mu}-J^{\tilde{\mu}}}_\infty  +\norm{J^{\tilde{\mu}}-J^{\mu}}_\infty +\norm{J^{\mu} - \scriptM J^{\mu}}_\infty    \allowdisplaybreaks
    \\
    &\leq  \alpha\norm{ \scriptM J^{\mu}-J^{\tilde{\mu}}}_\infty  +\frac{1}{1-\alpha} +\delta_2     \allowdisplaybreaks
    \\
    &\leq \alpha\norm{ \scriptM J^{\mu}-J^{\mu}}_\infty + \alpha\norm{J^{\mu}-J^{\tilde{\mu}}}_\infty +\frac{1}{1-\alpha}  +\delta_2  \allowdisplaybreaks
    \\
    &\leq  \alpha\delta_2  +\frac{1+\alpha}{1-\alpha}+\delta_2   \allowdisplaybreaks
    \\
    &=  (\frac{1}{1-\alpha}+\delta_2)(1+\alpha).
\end{align*}
\end{proof}

\begin{corollary} \label{corollary 1 unbiased}
We note that we can trace the steps of the proof of Lemma \ref{lemma 2 unbiased} when $\scriptM = I,$ where $I$ denotes the identity matrix to obtain the following bound:
\begin{align*}
    & \norm{ T_{\tilde{\mu}} J^{\mu} -  J^{\mu}}_\infty  
    \leq  \frac{1+\alpha}{1-\alpha}.
\end{align*} 
\end{corollary}

We now state our theorem which characterizes the role of function approximation on the convergence of approximate policy iteration with function approximation.
\begin{theorem} \label{UnbiasedLSThm}
When Assumptions \ref{assume 1}-\ref{assume 4} are satisfied, the sequence of iterates $V_k$ generated by Algorithm \ref{alg:LSUnbiased} has the following property, with probability 1:
\begin{align*}
    \limsup_{k \to \infty} \norm{V_k-J^*}_\infty &\leq  \frac{\delta_2}{1-\alpha^{H-1}}\\& + \frac{\alpha^{H-1} \Big((1+\alpha)(\delta_2+\frac{1}{1-\alpha})\Big) }{(1-\alpha)(1-\alpha^{H-1})},
\end{align*}
where 
$\delta_2 :=  \sup_{\mu}\norm{\scriptM J^{\mu}- J^{\mu}}_\infty.$
\end{theorem}

\begin{proof}
Using equation \eqref{eq:rewritetheta}, we can rewrite our iterates \eqref{eq:iteratesLSUnbiased} as follows:
\begin{align}
    V_{k+1} &= (1-\gamma_k)V_k + \gamma_k (\Phi \theta_{k+1}) \nonumber\\
    &= (1-\gamma_k)V_k + \gamma_k (\underbrace{\Phi(\Phi_{\scriptD }^\top \Phi_{\scriptD } )^{-1} \Phi_{\scriptD }^\top \scriptP }_{=: \scriptM} \hat{J}^{\mu_{k+1}}) \label{eq:defscriptM}\\
    &= (1-\gamma_k)V_k + \gamma_k \Big[\scriptM (J^{\mu_{k+1}}+w_{k})\Big].\nonumber
\end{align}


We define by $\tilde{\mu}_{k+1}$ the \textit{greedy} policy corresponding to vector $J_{k+1}$. In other words,
\begin{align}
    T J_{k} = T_{\tilde{\mu}_{k+1}} J_{k}. \label{eq:itergreedy}
\end{align}
Using this notation along with Lemma \ref{lemma 2 unbiased}, we establish the bounds given in the statement of the theorem.
\begin{align*}
    &TV_{k+1} - V_{k+1} = T_{\tilde{\mu}_{k+1}} V_{k+1} - V_{k+1} \allowdisplaybreaks\\
    &= T_{\tilde{\mu}_{k+1}} \Big[(1-\gamma_k)V_k + \gamma_k (\scriptM( J^{\mu_{k+1}} + w_k))\Big] \\&- \Big[ (1-\gamma_k)V_k + \gamma_k (\scriptM (J^{\mu_{k+1}} + w_k))\Big] \allowdisplaybreaks\\
    &= g_{\tilde{\mu}_{k+1}} + \alpha P_{\tilde{\mu}_{k+1}} \Big[ (1-\gamma_k)V_k + \gamma_k (\scriptM    (J^{\mu_{k+1}} + w_k))\Big] \\&- \Big[ (1-\gamma_k)V_k + \gamma_k (\scriptM (J^{\mu_{k+1}} + w_k))\Big] \allowdisplaybreaks\\
    &= g_{\tilde{\mu}_{k+1}} + \alpha P_{\tilde{\mu}_{k+1}} (1-\gamma_k)V_k + \alpha P_{\tilde{\mu}_{k+1}} \gamma_k \scriptM J^{\mu_{k+1}}  \\&-  (1-\gamma_k)V_k - \gamma_k \scriptM J^{\mu_{k+1}} +\gamma_k(\alpha P_{\tilde{\mu}_{k+1}} - I) \scriptM w_k \allowdisplaybreaks\\
    &= (1-\gamma_k) g_{\tilde{\mu}_{k+1}} + \alpha P_{\tilde{\mu}_{k+1}} (1-\gamma_k)V_k +\gamma_k g_{\tilde{\mu}_{k+1}} \\&+ \alpha P_{\tilde{\mu}_{k+1}} \gamma_k \scriptM J^{\mu_{k+1}}  -  (1-\gamma_k)V_k - \gamma_k \scriptM J^{\mu_{k+1}} + \gamma_k v_k \allowdisplaybreaks \\
    &= (1-\gamma_k) T V_k +\gamma_k g_{\tilde{\mu}_{k+1}} + \alpha P_{\tilde{\mu}_{k+1}} \gamma_k \scriptM J^{\mu_{k+1}}\\& -  (1-\gamma_k)V_k - \gamma_k \scriptM J^{\mu_{k+1}} + \gamma_k v_k \allowdisplaybreaks\\
    &= (1-\gamma_k)(T V_k -V_k) +\gamma_k T_{\tilde{\mu}_{k+1}} \scriptM J^{\mu_{k+1}} \\&-  \gamma_k \scriptM J^{\mu_{k+1}}  +\gamma_k v_k \allowdisplaybreaks\\
    &\leq (1-\gamma_k)( T V_k-V_k)  \\&+\gamma_k \norm{ T_{\tilde{\mu}_{k+1}}\scriptM J^{\mu_{k+1}} - \scriptM J^{\mu_{k+1}}}_\infty  e   + \gamma_k v_k  \allowdisplaybreaks\\
    &\leq (1-\gamma_k)(T V_k- V_k) + \gamma_k (\frac{1}{1-\alpha}+\delta_2)(1+\alpha) e  +\gamma_k v_k,
\end{align*} 
where $v_k := (\alpha P_{\tilde{\mu}_{k+1}} - I) \scriptM w_k$ and the last inequality follows from Lemma \ref{lemma 2 unbiased}. Additionally, 
Now, define $X_k := TV_k - V_k$ and we get the following:
\begin{align*}
    X_{k+1} \leq (1-\gamma_k)X_k +\gamma_k ((\frac{1}{1-\alpha}+\delta_2)(1+\alpha)e   + v_k).
\end{align*}
 Herein, when we refer to any convergence results, we mean convergence in an almost sure sense. Applying Assumption \ref{assume 2} and Lemma \ref{Lemma 1 Unbiased} in a similar way to the work of \cite{tsitsiklis2002convergence} gives the following: 
\begin{align*}\limsup_{k \to \infty} X_k \leq (\frac{1}{1-\alpha}+\delta_2)(1+\alpha)e.
\end{align*}

Thus, for every $\eps$ we have some $k(\eps)$ such that for $k > k(\eps),$ the following holds:
\begin{align*}
    &TV_k - V_k = X_k \leq \Big((1+\alpha)(\delta_2+\frac{1}{1-\alpha}) + \eps\Big) e \\
    &\implies TV_k \leq V_k + \Big((1+\alpha)(\delta_2+\frac{1}{1-\alpha}) + \eps\Big) e \\
        &\implies T^H V_k \leq T^{H-1} V_k\\& + \alpha^{H-1}\Big((1+\alpha)(\delta_2+\frac{1}{1-\alpha}) + \eps\Big) e \\
        &\implies T_{\mu_{k+1}}T^{H-1} V_k \leq T^{H-1} V_k \\&+ \alpha^{H-1}\Big((1+\alpha)(\delta_2+\frac{1}{1-\alpha})  + \eps\Big) e \\
    &\implies J^{\mu_{k+1}} \leq T^{H-1}V_k \\&+ \frac{\alpha^{H-1} \Big((1+\alpha)(\delta_2+\frac{1}{1-\alpha})  + \eps\Big)}{1-\alpha} e \\
    &\implies \scriptM J^{\mu_{k+1}} \leq T^{H-1}V_k \\&+ \underbrace{\Bigg[\delta_2 + \frac{\alpha^{H-1} \Big((1+\alpha)(\delta_2+\frac{1}{1-\alpha}) + \eps\Big)}{1-\alpha}\Bigg]}_{=: \kappa_{\eps}} e,
\end{align*}
where the last line follows from the definition of $\delta_2$ in the statement of Theorem \ref{UnbiasedLSThm}. Thus, for $k > k(\eps),$ we have:
\begin{align*}
    V_{k+1} &= (1-\gamma_k)V_k + \gamma_k (\scriptM(J^{\mu_{k+1}} + w_k)) \\
    &\leq (1-\gamma_k)V_k + \gamma_k (T^{H-1}V_k + \kappa_{\eps} e + \scriptM w_k). 
\end{align*}
Using a similar technique to the one in \cite{tsitsiklis2002convergence} and Lemma \ref{Lemma 1 Unbiased}, we get the following:
\begin{align*}
    \limsup_{k \to \infty} V_k \leq J^* + \frac{1}{1-\alpha^{H-1}}\kappa_{\eps} e.
\end{align*}
Since the above holds for all $\eps >0,$ we get that 
\begin{align*}
    &\limsup_{k \to \infty} V_k \\&\leq J^* + \frac{1}{1-\alpha^{H-1}}\Bigg[\delta_2 + \frac{\alpha^{H-1} \Big((1+\alpha)(\delta_2+\frac{1}{1-\alpha})\Big) }{1-\alpha}\Bigg]e.
\end{align*}
Furthermore, since $J^{\mu} \geq J^*$ for all policies $\mu,$ we obtain:
\begin{align*}
    V_{k+1} \nonumber&= (1-\gamma_k)V_k + \gamma_k (\scriptM(J^{\mu_{k+1}} + w_k))  \\
   \nonumber &\geq (1-\gamma_k)V_k + \gamma_k ( J^{\mu_{k+1}} - \delta_2 e + \scriptM w_k)  \\ 
    &\geq (1-\gamma_k)V_k + \gamma_k ( J^* - \delta_2  e + \scriptM w_k). 
\end{align*} 
Thus, applying Lemma \ref{Lemma 1 Unbiased}, we get that \begin{align*}
  \liminf_{k \to \infty} V_k \geq J^* - \delta_2 e,  
\end{align*} which implies Theorem \ref{UnbiasedLSThm}. 

\end{proof}
\begin{remark}\color{black}
Observe that Theorem \ref{UnbiasedLSThm} hinges on the bound in Lemma \ref{lemma 2 unbiased} where $\tilde{\mu}$ denotes the greedy policy and $\mu$ denotes the lookahead policy. In the special case of a 1-step lookahead, i.e., where the lookahead policy corresponds to the greedy policy, and feature vectors are unnecessary with $\scriptM = I,$ where $I$ is the identity matrix, we get that  
\begin{align*}
    & \norm{ T_{\tilde{\mu}}\scriptM J^{\mu} - \scriptM J^{\mu}}_\infty  
   =  0.
\end{align*} 
Using this fact and tracing the steps of the proof of Theorem \ref{UnbiasedLSThm} gives us the result of the work in \cite{tsitsiklis2002convergence} where the iterates converge almost surely to the optimal value function.
\end{remark}\color{black}
\section{Gradient Descent Algorithm} \label{sectionGDunbiased}
Computing $\theta_{k+1}$ in \eqref{eq:step 4 alg} at every iteration $k+1$ is sometimes infeasible as the computation involves inverting a matrix at each iteration. So, we introduce our second algorithm, Algorithm \ref{alg:GDAlgUnbiased}, in which we take $\eta_k$ steps of gradient descent towards $\theta_{k+1}$ of \eqref{eq:step 4 alg} at every iteration. We require the following to hold for the sequence $\eta_k:$
\begin{assumption} \label{assume 7}
\begin{align*}
\eta_k \to \infty. \end{align*}

\hfill $\diamond$
\end{assumption} We denote by $\beta$ the stepsize employed in the gradient descent. Our gradient descent iterates are given in \eqref{eq:iterthetaGD}. 
In order to obtain bounds on the convergence of our gradient descent iterates, we assume that $\beta$ is sufficiently small: 
\begin{assumption} \label{assume 5}
\begin{align*}\alpha' := \sup_k\norm{I - \beta \Phi_{\scriptD }^\top \Phi_{\scriptD }}_2 < 1.\end{align*}

\hfill $\diamond$
\end{assumption}
Finally, we assume that the noise sequences are bounded: 
\begin{assumption} \label{assume 6}
\begin{align*}
    \norm{w_k}_\infty \leq C_w.
\end{align*}

\hfill $\diamond$
\end{assumption}

\begin{algorithm}
\caption{Gradient Descent Function Approximation Algorithm With Unbiased Noise and Lookahead}
\label{alg:GDAlgUnbiased}
\textbf{Input}: $V_0, \gamma, \{\eta_k\}_{k=0}^\infty, \beta, \{\gamma_k\}_{k=0}^\infty,$ feature vectors $\{ \phi(i) \}_{i \in \scriptS}, \phi(i) \in \mathbb{R}^d$  and subsets $\scriptD  \subseteq \scriptS, k = 0, 1, \ldots.$ Here $\scriptD $ is the set of states at which we evaluate the current policy at iteration $k.$\\
\begin{algorithmic}[1] 
\STATE Let $k=0$.
\STATE Let $\mu_{k+1}$ be such that $T^H V_k = T_{\mu_{k+1}}T^{H-1}V_k$.\\
\STATE Compute $\hat{J}^{\mu_{k+1}}(i) =  J^{\mu_{k+1}}(i)+w_{k}(i)$ for $i \in \scriptD .$ \\ 
\STATE \label{step 4 unbiased} $\theta_{k+1, 0} := 0.$ For $\ell = 1, 2, \ldots, \eta_k,$ recursively compute the following:
\begin{align}
    \theta_{k+1, \ell} &= \theta_{k+1,\ell-1} - \beta  \nabla_\theta c(\theta;\hat{J}^{\mu_{k+1}})|_{\theta_{k+1,\ell-1}}, \label{eq:iterthetaGD}
\end{align} where
\begin{align*}
    c(\theta;\hat{J}^{\mu_{k+1}}) := \frac{1}{2}\sum_{i \in \scriptD } \Big( (\Phi \theta)(i) - \hat{J}^{\mu_{k+1}}(i) \Big)^2,
\end{align*}  \\
and $\Phi$ is a matrix whose rows are the feature vectors.\\
\STATE
\begin{align*}
V_{k+1} =(1-\gamma_k) V_k +\gamma_k (\Phi\theta_{k+1, \eta_k}).
\end{align*}
\\
\STATE Set $k \leftarrow k+1.$ Go to 2.
\end{algorithmic}
\end{algorithm}

We now present our convergence guarantees for Algorithm \ref{alg:GDAlgUnbiased}.
\begin{theorem}\label{GDUnbiasedThm}
Under assumptions \ref{assume 1}-\ref{assume 6}, the iterates obtained in Algorithm $\ref{alg:GDAlgUnbiased}$ almost surely have the following property:
\begin{align*}
&\limsup_{k \to \infty} \norm{V_k - J^*}_\infty \\&\leq \frac{1}{1-\alpha^{H-1}}\Bigg[\delta_2 + \frac{\alpha^{H-1} \Big((1+\alpha)(\delta_2+\frac{1}{1-\alpha})\Big)}{1-\alpha}\Bigg],
\end{align*}
where 
$\delta_2 :=  \sup_{\mu}\norm{\scriptM J^{\mu}- J^{\mu}}_\infty.$

\end{theorem}

\begin{proof}
We define $\hat{\theta}_{k+1}$ as follows:
\begin{align*}
    \hat{\theta}_{k+1} &:= \arg\min_{\theta} c(\theta;\hat{J}^{\mu_{k+1}}) \\
    &= (\Phi_{\scriptD }^\top \Phi_{\scriptD } )^{-1} \Phi_{\scriptD }^\top \scriptP  \hat{J}^{\mu_{k+1}} \\
    &= (\Phi_{\scriptD }^\top \Phi_{\scriptD } )^{-1} \Phi_{\scriptD }^\top \scriptP  (J^{\mu_{k+1}}+w_k).
\end{align*} Note that the above term is well-defined due to Assumption \ref{assume 3}. Additionally, note that 
\begin{align}
    \Phi \hat{\theta}_{k+1} = \scriptM  \hat{J}^{\mu_{k+1}} + \scriptM w_k,\label{eq:phihatthetaunbiased}
\end{align} from the definition of $\scriptM$ in equation \eqref{eq:defscriptM}.

We have that for every $k+1,$ $\theta_{k+1,\eta_k}$ is obtained by taking $\eta_k$ steps of gradient descent towards  $\hat{\theta}_{k+1}$ beginning from $\theta_{k+1,0} =0$. We show that the following holds for all $\ell \in 1, 2, \ldots, \eta_k$:
\begin{align*}
\norm{\theta_{k+1,\ell} - \hat{\theta}_{k+1}}_\infty \leq \alpha' \norm{\theta_{k+1,\ell-1} - \hat{\theta}_{k+1}}_\infty,
\end{align*} where $\alpha'$ is defined in Assumption \ref{assume 5}.

Recall that for a fixed $k+1,$ the iterates in Equation \eqref{eq:iterthetaGD} can be written as follows:
\begin{align*}
\theta_{k+1,\ell} &= \theta_{k+1,\ell-1} - \beta \nabla_\theta c(\theta;\hat{J}^{\mu_{k+1}})|_{\theta_{k+1,\ell-1}} \\&= \theta_{k+1,\ell-1} - \beta \Big( \Phi_{\scriptD }^\top \Phi_{\scriptD } \theta_{k+1,\ell-1}\\& - \Phi_{\scriptD }^\top \scriptP(J^{\mu_{k+1}}+w_k)\Big).
\end{align*}

Since 
\begin{align*}0 &= \nabla_\theta c(\theta;\hat{J}^{\mu_{k+1}})|_{\hat{\theta}_{k+1}}\\&= \Phi_{\scriptD }^\top \Phi_{\scriptD } \hat{\theta}_{k+1} - \Phi_{\scriptD }^\top \scriptP(J^{\mu_{k+1}}+w_k),
\end{align*}
we have the following:
\begin{equation*}
\begin{array}{lll}
\theta_{k+1,\ell} &=& \theta_{k+1,\ell-1} - \beta \Big( \Phi_{\scriptD }^\top \Phi_{\scriptD } \theta_{k+1,\ell-1} \\&&- \Phi_{\scriptD }^\top \Phi_{\scriptD } \hat{\theta}_{k+1} - \Phi_{\scriptD }^\top \scriptP(J^{\mu_{k+1}}+w_k) \\
&&+ \Phi_{\scriptD }^\top \scriptP(J^{\mu_{k+1}}+w_k)\Big)    \\
&=& \theta_{k+1,\ell-1} - \beta \Phi_{\scriptD }^\top \Phi_{\scriptD } (\theta_{k+1,\ell-1} - \hat{\theta}_{k+1}).
\end{array}
\end{equation*}
Subtracting $\hat{\theta}_{k+1}$ from both sides gives:
\begin{align*}
&\theta_{k+1,\ell} - \hat{\theta}_{k+1} \\&=  \theta_{k+1,\ell-1} - \hat{\theta}_{k+1} - \beta \Phi_{\scriptD }^\top \Phi_{\scriptD } (\theta_{k+1,\ell-1}- \hat{\theta}_{k+1})\\&= (I - \beta \Phi_{\scriptD }^\top \Phi_{\scriptD }) (\theta_{k+1,\ell-1} - \hat{\theta}_{k+1}) \\
\implies  &\Phi\theta_{k+1,\ell} -\Phi \hat{\theta}_{k+1} \\&= (I - \beta \Phi_{\scriptD }^\top \Phi_{\scriptD }) (\Phi \theta_{k+1,\ell-1} - \Phi\hat{\theta}_{k+1}).
\end{align*}

Thus, 
\begin{align*}
\nonumber&\norm{\Phi\theta_{k+1,\ell} - \Phi\hat{\theta}_{k+1}}_\infty \nonumber \\&=  \norm{(I - \beta \Phi_{\scriptD }^\top \Phi_{\scriptD }) (\Phi\theta_{k+1,\ell-1} - \Phi\hat{\theta}_{k+1})}_\infty \\&=  \norm{I - \beta \Phi_{\scriptD }^\top \Phi_{\scriptD }}_\infty \norm{\Phi \theta_{k+1,\ell-1} - \Phi\hat{\theta}_{k+1}}_\infty \\ \nonumber
&\leq  \norm{I - \beta \Phi_{\scriptD }^\top \Phi_{\scriptD }}_2 \norm{\Phi \theta_{k+1,\ell-1} - \Phi\hat{\theta}_{k+1}}_\infty\\& \leq \underbrace{\sup_k\norm{I - \beta \Phi_{\scriptD }^\top \Phi_{\scriptD }}_2}_{= \alpha'} \norm{\Phi \theta_{k+1,\ell-1} - \Phi\hat{\theta}_{k+1}}_\infty, 
\end{align*} where $\alpha'$ is defined in Assumption \ref{assume 5}.

Thus, 
\begin{align}
\nonumber\norm{\Phi\theta_{k+1,\eta_k} - \Phi\hat{\theta}_{k+1}}_\infty \nonumber &\leq \alpha'^{\eta_k} \norm{\Phi \theta_{k+1,0} - \Phi\hat{\theta}_{k+1}}_\infty \nonumber\\
&=\alpha'^{\eta_k} \norm{ \Phi\hat{\theta}_{k+1}}_\infty \label{eq:iternu},
\end{align} where the equality follows from the definition of $\theta_{k+1,0}$ in Step \ref{step 4 unbiased} of Algorithm \ref{alg:GDAlgUnbiased}.

We will now bound $\norm{ \Phi\hat{\theta}_{k+1}}_\infty$ for all $k+1$.

First, we note that from the definitions of $\scriptM$ and $\delta_2$ in \eqref{eq:defscriptM} and Theorem \ref{UnbiasedLSThm}, respectively, that the following holds: 

The following holds:
\begin{align*}
 \norm{\Phi\hat{\theta}_{k+1}}_\infty &= \norm{\Phi (\Phi_{\scriptD }^\top \Phi_{\scriptD } )^{-1} \Phi_{\scriptD }^\top \scriptP  (J^{\mu_{k+1}}+w_k)}_\infty \\
 &= \norm{\scriptM (J^{\mu_{k+1}}+w_k)}_\infty \\
 &\leq \norm{\scriptM J^{\mu_{k+1}}}_\infty + \norm{\scriptM w_k}_\infty \\
 &\leq \norm{\scriptM J^{\mu_{k+1}}-J^{\mu_{k+1}}+J^{\mu_{k+1}}}_\infty+ C_w \norm{\scriptM}_\infty \\
 &\leq \delta_2 + \frac{1}{1-\alpha} + C_w \norm{\scriptM}_\infty,
\end{align*} where the second to last inequality follows from Assumption \ref{assume 7}.

From \eqref{eq:iternu}, we get that:
\begin{align}
\nonumber\norm{\Phi\theta_{k+1,\eta_k} -\Phi \hat{\theta}_{k+1}}_\infty \nonumber 
&\leq \alpha'^{\eta_k} \Bigg(\delta_2 + \frac{1}{1-\alpha} + C_w \norm{\scriptM}_\infty \Bigg).
\end{align}

Thus, under Assumption \ref{assume 6}, for any $\eps_1 >0,$ we get that there exists some $k(\eps_1)$ such that for all $k > k(\eps_1)$:
\begin{align}
\nonumber\norm{\Phi\theta_{k+1,\eta_k} -\Phi \hat{\theta}_{k+1}}_\infty \nonumber &\leq \eps_1.
\end{align}

Now, we have the following for $k > k(\eps_1)$ using :
\begin{align*}
    &\norm{\Phi \theta_{k+1, \eta_k} - (J^{\mu_{k+1}}+\scriptM w_k)}_\infty \\&=\norm{\Phi \theta_{k+1, \eta_k} -\Phi \hat{\theta}_{k+1}+\Phi \hat{\theta}_{k+1}- (J^{\mu_{k+1}}+\scriptM w_k)}_\infty \\
    &= \norm{\Phi \theta_{k+1, \eta_k} -\Phi \hat{\theta}_{k+1}}_\infty +\norm{\Phi \hat{\theta}_{k+1}- (J^{\mu_{k+1}}+\scriptM w_k)}_\infty\\
    &\leq \eps_1  +\norm{\scriptM(J^{\mu_{k+1}}+ w_k)- (J^{\mu_{k+1}}+\scriptM w_k)}_\infty\\&
    \leq \eps_1 + \delta_2.
\end{align*}

When we denote $\tilde{\mu}_{k+1}$ as the greedy policy in \eqref{eq:itergreedy}, we have the following:
\begin{align*}
    &TV_{k+1} - V_{k+1} \\&= T_{\tilde{\mu}_{k+1}} V_{k+1} - V_{k+1} \allowdisplaybreaks\\ 
    &= T_{\tilde{\mu}_{k+1}} \Big[ (1-\gamma_k) V_k +\gamma_k (\Phi\theta_{k+1, \eta_k})\Big]\\&-\Big[(1-\gamma_k) V_k +\gamma_k (\Phi\theta_{k+1, \eta_k})\Big] \allowdisplaybreaks\\ 
    &= g_{\tilde{\mu}_{k+1}} + \alpha P_{\tilde{\mu}_{k+1}} (1-\gamma_k)V_k \\&+ \alpha P_{\tilde{\mu}_{k+1}} \gamma_k \Phi\theta_{k+1, \eta_k}  -  (1-\gamma_k)V_k - \gamma_k \Phi\theta_{k+1, \eta_k}\allowdisplaybreaks\\
    &= (1-\gamma_k)g_{\tilde{\mu}_{k+1}} + \alpha P_{\tilde{\mu}_{k+1}} (1-\gamma_k)V_k + \gamma_k g_{\tilde{\mu}_{k+1}}\\&+ \alpha P_{\tilde{\mu}_{k+1}} \gamma_k \Phi\theta_{k+1, \eta_k}  -  (1-\gamma_k)V_k  - \gamma_k \Phi\theta_{k+1, \eta_k}\allowdisplaybreaks \\ 
    &= (1-\gamma_k)(T V_k -V_k) +\gamma_k g_{\tilde{\mu}_{k+1}} + \gamma_k \alpha P_{\tilde{\mu}_{k+1}}  \Phi\theta_{k+1, \eta_k} \\& -  \gamma_k  \Phi\theta_{k+1, \eta_k} \allowdisplaybreaks\\ 
    &\leq  (1-\gamma_k)(T V_k -V_k) +\gamma_k g_{\tilde{\mu}_{k+1}}\\& + \gamma_k \alpha P_{\tilde{\mu}_{k+1}} (J^{\mu_{k+1}}+\scriptM w_k +(\delta_2+\eps_1)e)  \\&-  \gamma_k  (J^{\mu_{k+1} }+\scriptM w_k  - (\delta_2 + \eps_1)e)\allowdisplaybreaks \\ 
    &\leq  (1-\gamma_k)(T V_k -V_k) +\gamma_k T_{\tilde{\mu}_{k+1}} J^{\mu_{k+1}}-\gamma_k J^{\mu_{k+1}}\\&+\gamma_k  (\alpha P_{\tilde{\mu}_{k+1}}-I)\scriptM w_k +\gamma_k(  \alpha P_{\tilde{\mu}_{k+1}}+I)(\delta_2+\eps_1)\allowdisplaybreaks \\ 
    &\leq  (1-\gamma_k)(T V_k -V_k) +\gamma_k \norm{T_{\tilde{\mu}_{k+1}} J^{\mu_{k+1}}- J^{\mu_{k+1}}}_\infty e\\&+\gamma_k  (\alpha P_{\tilde{\mu}_{k+1}}-I)\scriptM w_k +\gamma_k (\delta_2+\eps_1)\norm{  \alpha P_{\tilde{\mu}_{k+1}}+I}_\infty e\allowdisplaybreaks \\ 
    &\overset{(a)}{\leq}(1-\gamma_k)(T V_k -V_k) +\gamma_k  v_k \\&+\gamma_k (1+\alpha)(\delta_2 +\eps_1+\frac{1}{1-\alpha})e
, 
\end{align*} where $v_k := (\alpha P_{\tilde{\mu}_{k+1}}-I)\scriptM w_k$ and $(a)$ follows from Corollary \ref{corollary 1 unbiased}.

Now, define $X_k = TV_k - V_k$ and we get the following:
\begin{align*}
    X_{k+1} \leq (1-\gamma_k)X_k +\gamma_k ((1+\alpha)(\delta_2+\eps_1+\frac{1}{1-\alpha})e  + v_k).
\end{align*}
Thus, using Assumption \ref{assume 2} and Lemma \ref{Lemma 1 Unbiased}, it can be easily shown that
\begin{align*}
    \limsup_{k \to \infty} X_k \leq (1+\alpha)(\delta_2+\eps_1+\frac{1}{1-\alpha})e .
\end{align*}

Thus, for every $\eps_2 < \eps_1,$ we have some $k(\eps_2) > k(\eps_1)$ such that for $k > k(\eps_2),$ the following holds:
\begin{align*}
    &TV_k - V_k = X_k \leq \Big((1+\alpha)(\delta_2+\eps_1+\frac{1}{1-\alpha})+ \eps_2\Big) e \\ \allowdisplaybreaks
    &\implies TV_k \leq V_k + \Big((1+\alpha)(\delta_2+\eps_1+\frac{1}{1-\alpha}) + \eps_2\Big) e \\ \allowdisplaybreaks
        &\implies T^H V_k \leq T^{H-1} V_k\\& + \alpha^{H-1}\Big((1+\alpha)(\delta_2+\eps_1+\frac{1}{1-\alpha})e + \eps_2\Big) e \\ \allowdisplaybreaks
        &\implies T_{\mu_{k+1}}T^{H-1} V_k \leq T^{H-1} V_k \\&+ \alpha^{H-1}\Big((1+\alpha)(\delta_2+\eps_1+\frac{1}{1-\alpha})e + \eps_2\Big) e \\ \allowdisplaybreaks
    &\implies J^{\mu_{k+1}} \leq T^{H-1}V_k \\&+ \frac{\alpha^{H-1} \Big((1+\alpha)(\delta_2+\eps_1+\frac{1}{1-\alpha})e + \eps_2\Big)}{1-\alpha} e \\ \allowdisplaybreaks
    &\implies  \Phi\theta_{k+1,\eta_k} \leq T^{H-1}V_k + v_k \\&+ \underbrace{\Bigg[\delta_2+\eps_1 + \frac{\alpha^{H-1} \Big((1+\alpha)(\delta_2+\eps_1+\frac{1}{1-\alpha}) + \eps_2\Big)}{1-\alpha}\Bigg]}_{=: \kappa_{\eps_2,\eps_1}} e. 
\end{align*}

Thus, for $k > k(\eps_2),$ we have:
\begin{align}
    V_{k+1} &\leq (1-\gamma_k)V_k + \gamma_k (T^{H-1}V_k + \kappa_{\eps_2,\eps_1} e + v_k). 
\end{align}
Applying Lemma \ref{Lemma 1 Unbiased}, we get the following:
\begin{align*}
    \limsup_{k \to \infty} V_k \leq J^* + \frac{\kappa_{\eps_2,\eps_1}}{1-\alpha^{H-1}} e.
\end{align*}
Since the above holds for all $\eps_1, \eps_2 >0,$ we get that 
\begin{align*}
    &\limsup_{k \to \infty} V_k \leq \\&J^* + \frac{1}{1-\alpha^{H-1}}\Bigg[\delta_2 + \frac{\alpha^{H-1} \Big((1+\alpha)(\delta_2+\frac{1}{1-\alpha})\Big)}{1-\alpha}\Bigg]e.
\end{align*}



Furthermore, since $J^{\mu} \geq J^*$ for all policies $\mu,$ we obtain the following for $k > k(\eps_1)$:
\begin{align*}
    V_{k+1} \nonumber&= (1-\gamma_k) V_k +\gamma_k (\Phi \theta_{k+1, \eta_k}) \\
    &\geq (1-\gamma_k)V_k + \gamma_k ( J^{\mu_{k+1}} - \delta_2  e -\eps_1 + v_k) \\
    &\geq (1-\gamma_k)V_k + \gamma_k ( J^* - \delta_2   e -\eps_1 + v_k).
\end{align*} 
Thus, since the above holds for all $\eps_1>0,$ applying Lemma \ref{Lemma 1 Unbiased}, \begin{align*} 
    \liminf_{k \to \infty} V_k \geq J^* - \delta_2 e,
\end{align*} and we get Theorem \ref{GDUnbiasedThm}.
\end{proof}






\section{Conclusion}

We study the convergence of function approximation based approximate policy iteration algorithms with stochastic approximation techniques when lookahead and gradient descent are involved. The upper bounds on asymptotic error decrease exponentially with increasing amount of lookahead and depend only on the feature vectors used for function approximation rather than the size of the state space. Additionally, while we assume that the noise must be bounded and unbiased, our upper bounds for asymptotic error do not depend on the nature of the noise. We outline several directions for further work:
\begin{itemize}
    \item In our main proofs, we assume that we obtain $d$ trajectories beginning at states in $\scriptD$ at every iteration $k$ to estimate $J^{\mu_k}$. It is possible to generalize this result in several directions: (i) it is easy to relax this assumption to allow the set of states at which we evaluate the policy to vary with each iteration if we assume that these states are chosen independent of past history. (ii) A more practical assumption would be one in which the states are simply the states visited by one trajectory under the current policy. Our results can also be extended to this with additional terms in the error bounds which can be controlled using multi-step lookahead. 
    \item In the gradient descent algorithm, we have assumed that $\eta_k\rightarrow\infty.$ An alternative is to fix $\eta_k$ to be a constant and obtain asymptotic error bounds. This is a straightforward extension of our results.
    \item We have obtained our results with the use of linear function approximators. However, it has recently been suggested in works on the NTK analysis of neural networks that neural networks could be approximated as linear combinations of basis functions. One direction of further work could involve an extension of the current work to include the use of neural networks in addition to the linear function approximators \cite{ji2019polylogarithmic,satpathi}.
\end{itemize}

\section*{ACKNOWLEDGMENT}
The research presented here was supported by the following grants: ONR N00014-19-1-2566, NSF CCF 17-04970, NSF CCF 1934986, and ARO W911NF-19-1-0379.

\bibliographystyle{plain} 
\bibliography{aaai22.bib} 


\end{document}